\documentclass[letterpaper]{article} % DO NOT CHANGE THIS
\usepackage{aaai2026}  % DO NOT CHANGE THIS
\usepackage{times}  % DO NOT CHANGE THIS
\usepackage{helvet}  % DO NOT CHANGE THIS
\usepackage{courier}  % DO NOT CHANGE THIS
\usepackage[hyphens]{url}  % DO NOT CHANGE THIS
\usepackage{graphicx} % DO NOT CHANGE THIS
\usepackage{natbib}  % DO NOT CHANGE THIS AND DO NOT ADD ANY OPTIONS TO IT
\usepackage{caption} % DO NOT CHANGE THIS AND DO NOT ADD ANY OPTIONS TO IT
\usepackage{amsfonts}       % blackboard math symbols
\usepackage{microtype}      % microtypography (possibly slow)
\usepackage{subcaption}
\usepackage{amsmath}
\usepackage{amsthm}
\usepackage{amssymb}

\usepackage[ruled,vlined,linesnumbered]{algorithm2e}
\usepackage{cleveref}

\newtheorem{definition}{Definition}

\newtheorem{theorem}{Theorem}

\mathchardef\mhyphen="2D

\usepackage{xcolor}
\usepackage[inline]{enumitem}
\usepackage{booktabs}       % professional-quality tables
\usepackage{multirow} 
\usepackage{array}   
\usepackage{siunitx} 
\usepackage{makecell}

\title{Scalable Solution Methods for Dec-POMDPs with Deterministic Dynamics}
\author{
    Yang You\textsuperscript{\rm 1},
    Alex Schutz\textsuperscript{\rm 2},
    Zhikun Li\textsuperscript{\rm 2},
    Bruno Lacerda\textsuperscript{\rm 2},
    Robert Skilton\textsuperscript{\rm 1},
    Nick Hawes\textsuperscript{\rm 2}
}
\affiliations{
    \textsuperscript{\rm 1} United Kingdom Atomic Energy Authority\\
    \textsuperscript{\rm 2} Oxford Robotics Institute, University of Oxford\\
}

\usepackage{bibentry}
\begin{document}

\maketitle

\begin{abstract}
Many high-level multi-agent planning problems, such as multi-robot navigation and path planning, can be modeled with deterministic actions and observations.
In this work, we focus on such domains and introduce the class of Deterministic Decentralized POMDPs (Det-Dec-POMDPs)—a subclass of Dec-POMDPs with deterministic transitions and observations given the state and joint actions.
We then propose a practical solver, \textit{Iterative Deterministic POMDP Planning} (IDPP), based on the classic Joint Equilibrium Search for Policies framework, specifically optimized to handle large-scale Det-Dec-POMDPs that existing Dec-POMDP solvers cannot handle efficiently.
\end{abstract}

\section{Introduction}

Decentralized partially observable Markov decision processes (Dec-POMDPs) are widely used to model multi-agent decision-making under uncertainty and partial observability, where each agent acts based solely on its own action-observation history.
While highly expressive, Dec-POMDPs are difficult to solve.
Even for finite horizons, solving Dec-POMDPs optimally is NEXP-complete \cite{bernsteinComplexityDecentralizedControl2002}.
To reduce this complexity, \citeauthor{besse2009quasi} introduced the Quasi-Deterministic Dec-POMDP (QDet-Dec-POMDP) \cite{besse2009quasi}, which assumes deterministic transitions but retains stochastic observations.
While this quasi-deterministic structure simplifies some aspects of Dec-POMDPs, the stochastic observations can still significantly hinder scalability.

Motivated by the observation that in many real-world robotic mission planning scenarios, high-level decision-making often involves both deterministic action outcomes and effectively deterministic observations, our first contribution is to propose a further simplification of existing models: the Deterministic Decentralized POMDP (Det-Dec-POMDP).
In a Det-Dec-POMDP, uncertainty exists only in the initial state distribution, while both the transition and observation models are fully deterministic.
This model can be seen as a natural extension of deterministic POMDPs (Det-POMDPs) \cite{littmanAlgorithmsSequentialDecisionmaking1996} to the multi-agent setting.
Leveraging this fully deterministic structure, our second contribution introduces a practical solver called \textit{Iterative Deterministic POMDP Planning} (IDPP) that specifically optimized for solving large Det-Dec-POMDPs.
IDPP improves upon prior JESP frameworks \cite{nair2003taming, you2021solving, you2023monte} by exploiting deterministic system dynamics. In each iteration, it invokes an efficient Det-POMDP planner \cite{schutz2025finitestatecontrollerbasedoffline} to compute each agent’s best-response policy, ultimately converging to a Nash equilibrium policy set, as shown in \Cref{fig:solution-fig}.
Although simple, this practical enhancement significantly improves scalability and enables efficient planning in large Det-Dec-POMDPs that existing Dec-POMDP solvers struggle to handle.
Moreover, to facilitate future research on algorithm scalability, we introduce two scalable Det-Dec-POMDP benchmarks that scale to millions of states and thousands of observations per agent.
\begin{figure}[t!]
    \centering
\includegraphics[width=0.95\linewidth]{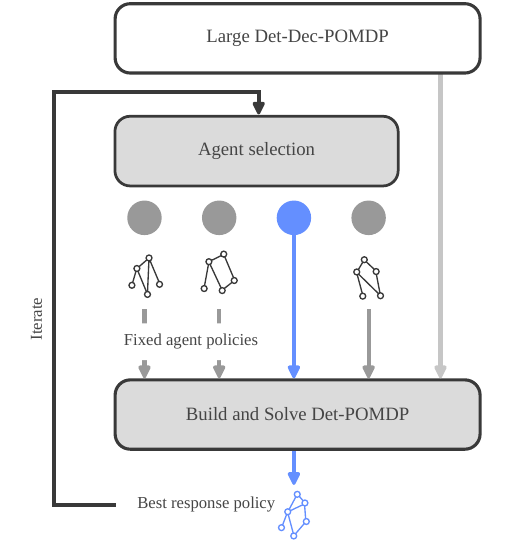}
    \caption{The approach of IDPP for solving large Det-Dec-POMDPs. The joint policy is decomposed into individual agent policies, which are initialized using a heuristic. Policies are improved using an iterative best-response process.}
    \label{fig:solution-fig}
\end{figure}
\section{Related Work}

\paragraph{Applications with Deterministic POMDPs.}

Recent advances in efficient algorithms have driven significant progress in Deterministic POMDP applications.
One particularly useful class involves environments where structural elements are initially uncertain but can be deterministically observed during execution. For example, a robot may have a prior over an object's location, or a waiter robot may choose to bring the most likely item before confirming a request.
These problems can be made deterministic at the action and observation level through appropriate abstractions, such as checking whether a door is open or detecting an object using a reliable classifier. 
Such examples appear in robot forest path planning \cite{schutz2025finitestatecontrollerbasedoffline} and robots navigation under centralized control \cite{stadler2023approximating}. 
Beyond robotics, areas such as circuit synthesis, sorting networks, and communication protocols can also be modeled as Det-POMDPs \cite{bonetDeterministicPOMDPsRevisited2009}. 
%
% However, unlike in the Det-POMDP case, practical solvers for large Det-Dec-POMDPs are still lacking.

\paragraph{Solving General Dec-POMDPs.} 
State-of-the-art Dec-POMDP planning methods broadly fall into three categories.
The first frames Dec-POMDPs as inference problems, estimating optimal parameters for each agent’s policy—often represented as finite-state controllers (FSCs) \cite{amatoOptimizingFixedsizeStochastic2010, pajarinen2011efficient,pajarinen2011periodic,kumar2012anytime,kumar2015probabilistic,song2016solving}. These approaches are well-suited for infinite-horizon problems and can produce compact policies, but the underlying non-convex optimization often suffers from poor local optima, limiting solution quality.
The second transforms Dec-POMDPs into centralized sequential decision problems by constructing sufficient statistic spaces such as occupancy or information states \cite{szerMAAHeuristicSearch2012,macdermed2013point,dibangoye2016optimally}. This enables the use of powerful POMDP solvers to compute optimal joint policies, which are then decomposed into decentralized agent policies.
%
% Although this offers bounded errors and sometimes $\epsilon$-optimality guarantees, 
However, the exponential growth of the statistic space may limit scalability.
The third category relaxes global optimality by seeking Nash equilibrium solutions \cite{nair2003taming, bernstein2005bounded,bernstein2009policy, you2021solving,you2023monte}, where each agent’s policy is a best response to fixed policies of others.
This approach often scales better in infinite-horizon problems by reducing the problem to sequential single-agent POMDPs, at the cost of lacking optimality guarantees.

Another major line of work is multi-agent reinforcement learning (MARL), which tackles Dec-POMDPs through a learning perspective \cite{sunehag2017value,Rashid2018QMIX,Yu2021MAPPO}. To handle partial observability, they often incorporate recurrent architectures \cite{hochreiter1997long} that maintain internal memory of action-observation histories.
However, these methods often require extensive training time and struggle with sample inefficiency, especially when the agents' reward signals are sparse.

% While there are some works addressing deterministic POMDPs \cite{bonetDeterministicPOMDPsRevisited2009,schutz2025finitestatecontrollerbasedoffline}, to the best of our knowledge there are no planning methods that specifically leverage determinism in both actions and observations for Dec-POMDPs.

% On the other hand, while there are substantial works developed for deterministic POMDPs \cite{bonetDeterministicPOMDPsRevisited2009,schutz2025finitestatecontrollerbasedoffline}, such as those with deterministic actions and observations, there is a notable rarity in Dec-POMDP planning methods that specifically leverage deterministic features.

\paragraph{Motivation and Our Contribution.}

We observe that many applications of single-agent Det-POMDPs can naturally extend to decentralized multi-agent settings, forming Det-Dec-POMDPs—for example, generalizing single-robot navigation to multi-robot navigation in a forest. However, unlike the single-agent case, scalable solvers for large Det-Dec-POMDPs remain underdeveloped, and general Dec-POMDP methods often struggle with such domains. This lack of efficient methods may, in turn, limit the practical adoption of Det-Dec-POMDPs. In this work, beyond formalizing Det-Dec-POMDPs, we propose a simple yet practical solution method to enable future applications.

\section{Background}

\subsection{Partially Observable Decision Models}

A Partially Observable Markov Decision Process (POMDP) models a decision-making problem where the agent cannot directly observe the true underlying state. 
A POMDP is formally defined as a tuple \(\langle\mathcal{S}, \mathcal{A}, \Omega, \mathcal{T}, \mathcal{O}, \mathcal{R}, \gamma, b_0\rangle\), where \(\mathcal{S}\) denotes the set of states, \(\mathcal{A}\) the set of actions, and \(\Omega\) the set of observations. 
The transition function \(\mathcal{T}(s, a, s') = \Pr(s' \mid s, a)\) specifies the probability of reaching state \(s'\) after taking action \(a\) in state \(s\), while the observation function \(\mathcal{O}(a, s', o) = \Pr(o \mid s', a)\) gives the probability of observing \(o\) after arriving at state \(s'\) via action \(a\). 
The reward function \(\mathcal{R}(s, a)\) defines the immediate reward received for taking action \(a\) in state \(s\), and \(\gamma \in [0, 1)\) is the discount factor that models the agent’s preference for immediate rewards over future ones.  
The initial belief \(b_0\) denotes the initial state distribution.
In POMDPs, at each timestep, the agent updates a belief (a probability distribution over \(\mathcal{S}\)) based on the actions taken and observations received.
The goal of solving a POMDP is to find a policy that maps beliefs to actions in order to maximize expected discounted rewards over time.

A Decentralized POMDP (Dec-POMDP) extends POMDPs to cooperative multi-agent settings. In a Dec-POMDP, multiple agents jointly control the environment, each making decisions based on their local action-observations. 
Agents aim to coordinate implicitly through their policies to maximize a shared cumulative reward.
Planning in Dec-POMDPs is significantly more challenging than POMDPs due to this decentralized feature and the exponential growth of joint policy spaces.

Recent work has studied subclasses of POMDPs and Dec-POMDPs with deterministic or partially deterministic dynamics \cite{besse2009quasi, bonetDeterministicPOMDPsRevisited2009}. Deterministic POMDPs (Det-POMDPs) have deterministic state transition and observation functions, while Quasi-Deterministic POMDPs (QDET-Dec-POMDPs) feature deterministic transitions but stochastic observations. However, deterministic Dec-POMDPs (Det-Dec-POMDPs), which naturally extend Det-POMDPs to multi-agent settings, have not been specifically studied to the best of our knowledge.

\subsection{Finite-State Controllers}

A Finite-State Controller (FSC) is a compact representation of a policy for agents in POMDPs and Dec-POMDPs.
Instead of mapping full histories or beliefs to actions, an FSC encodes a policy as a finite automaton defined by a tuple $(\mathcal{N}, \psi, \eta, n^0)$, where:
$\mathcal{N}$ is a finite set of controller nodes (internal states);
$\psi: \mathcal{N} \rightarrow \mathcal{A}$ is the action selection function;
$\eta: \mathcal{N} \times \mathcal{O} \rightarrow \mathcal{N}$ is the node transition function based on observations, and
$n^0 \in \mathcal{N}$ is the initial node.
At each time step, the agent selects an action \textit{deterministically} $a = \psi(n)$ based on its current node $n \in \mathcal{N}$, and upon receiving an observation $o \in \mathcal{O}$, it transitions to a new node $n' = \eta(n, o)$ \textit{deterministically} .
%
% Note that, there are also \textit{stochastic} FSCs where action selection and node transition function are modeled with probability distributions. 
% %
% But in this paper, we stick to the deterministic version of the FSC.
% \alex{\cite{oliehoek2008optimal} "a finite-horizon Dec-POMDP has at least one optimal pure joint policy" - Does this extend to infinite horizon?}
% \alex{In this paper we stick to the deterministic version of the FSC for simplicity without sacrificing optimality \cite{oliehoek2008optimal}.}

FSCs are widely used in infinite-horizon planning \cite{baiMonteCarloValue2011, NIPS2011_eefc9e10, you2021solving, you2023monte} due to their ability to represent policies compactly and operate without tracking the full belief state or history. 
Note that, there are also \textit{stochastic} FSCs where action selection and node transition function are modeled with probability distributions. 
In this paper we stick to the deterministic version of the FSC for simplicity without sacrificing optimality \cite{oliehoek2008optimal}.
% In this work, we also use FSCs to represent policies in deterministic Dec-POMDPs.

\subsection{Finding Nash-Equilibrium Solutions}

Joint Equilibrium-based Search for Policies (JESP) approaches \cite{nair2003taming, you2021solving, you2023monte} aim to find Nash Equilibrium solutions by iteratively computing one agent's best-response policy while fixing the policies of all other agents.
All JESP methods share a common algorithmic structure, given in \Cref{alg:JESP}:

\begin{algorithm}
\caption{General JESP Framework}
\label{alg:JESP}
\KwIn{Initial policies for all agents}
\KwOut{A Nash equilibrium policy set}
\While{policies have not converged}{
    Select the current optimizing agent $i$\;
    Fix other agents' policies $\pi_{\neq i}$ and construct agent $i$'s best-response model $\text{POMDP}_{\text{BR},i}$\;
    Solve $\text{POMDP}_{\text{BR},i}$ and update agent $i$'s policy $\pi_i$\;
}
\end{algorithm}

Infinite-Horizon JESP (InfJESP) \cite{you2021solving} extends JESP to infinite-horizon Dec-POMDPs by representing each agent's policy as a finite-state controller and constructing the best-response POMDP accordingly.
Each agent's best-response model $\text{POMDP}_{\text{BR},i}$ is solved using SARSOP \cite{kurniawati2008sarsop}, enabling planning over infinite horizons.
In InfJESP, agent $i$'s $\text{POMDP}_{\text{BR},i}$ uses an extended state space $e^t \in \mathcal{E}$ containing:
\begin{itemize*}
    \item $s^t$, the current environment state;
    \item $n_{\neq i}^t = \langle n_j^t \rangle_{j \neq i}$, the current nodes of other agents' FSCs;
    \item $\tilde{o}_i^t$, agent $i$'s current observation.
\end{itemize*}
This extended state enables defining a valid best-response POMDP with the following dynamics:
\begin{align*}
  &  \mathcal{T}_e(e^t, a^t_i, e^{t+1})  = Pr(e^{t+1}| e^t, a^t_i) \\
  & \quad = \sum_{o^{t+1}_{\neq i}}T(s^t, \langle  \psi_{\neq i}(n_{\neq i}^{t}), a^t_{i} \rangle, s^{t+1})
  \cdot \mathbf{1}_{ n_{\neq i}^{t+1} = \eta_{\neq i}(n_{\neq i}^{t}, o^{t+1}_{\neq i}) } \\
  & \qquad \cdot O(s^{t+1}, \langle \psi_{\neq i}(n_{\neq i}^{t}), a^t_{i} \rangle, \langle o^{t+1}_{\neq i}, o^{t+1}_{i} \rangle),
  \\ %
  & \mathcal{O}_e(a^t_i, e^{t+1}_i, o^{t+1}_i)
  = Pr( o^{t+1}_i | a^t_i, e^{t+1}_i) \\
  & \quad = Pr( o^{t+1}_i | a^t_i, \langle s^{t+1}, n^{t+1}_{\neq i}, \tilde o^{t+1}_i \rangle) %
  = \mathbf{1}_{o^{t+1}_i = \tilde  o^{t+1}_i},
  \\
  & r_e(e^t, a^t_i) =  r(s^t, a^t_i, \psi_{\neq i}(n_{\neq i}^{t}) ).
\end{align*}
where:
\begin{itemize}
    \item $\psi_{\neq i}(n_{\neq i}^t) = \langle \psi_j(n_j^t) \rangle_{j \neq i}$ denotes the other agents' action selections.
    \item $\eta_{\neq i}(n_{\neq i}^t, o_{\neq i}^{t+1}) = \langle \eta_j(n_j^t, \tilde{o}_j^{t+1}) \rangle_{j \neq i}$ denotes the other agents' FSC node transitions.
\end{itemize}

Note that, in this expression, the extended state $e^t$ explicitly includes agent $i$'s current observation $\tilde{o}_i^t$, which results in a deterministic observation function $O_e$ when constructing agent $i$'s best-response POMDP for any Dec-POMDP.
In MCJESP \cite{you2023monte}, the best-response model $\text{POMDP}_{\text{BR},i}$ is represented implicitly via a generative model $G_{\text{POMDP}_{\text{BR},i}}$, rather than explicitly enumerating all transitions. 
Agent \(i\)'s FSC policy is then optimized through Monte Carlo search \cite{silver2010monte}; specifically, each FSC node is associated with a belief, and MCJESP uses POMCP to compute the best action for that node, updating agent \(i\)'s FSC in a node-by-node manner.
This allowing MCJESP to scale to larger Dec-POMDPs by avoiding the computational overhead associated with explicit dynamics representation.

\section{Deterministic Dec-POMDPs}

This section formally defines the deterministic Dec-POMDP (Det-Dec-POMDP), an extension of the single-agent Det-POMDP to decentralized multi-agent settings.

\begin{definition} A Det-Dec-POMDP is a tuple $\langle \mathcal{I}, \mathcal{S}, \mathcal{A}, \Omega, \mathcal{T}, \mathcal{O}, \mathcal{R}, \gamma, b_0 \rangle$, where:
\begin{itemize*}
\item $\mathcal{I}$ is the finite set of agents, with $i \in \mathcal{I}$;
\item $\mathcal{S}$ is the set of states, $s \in \mathcal{S}$; 
\item $\mathcal{A} = \times_{i \in \mathcal{I}} \mathcal{A}_i$ is the set of joint actions, where $\mathcal{A}_i$ is the action set of agent $i$; 
\item $\Omega = \times_{i \in \mathcal{I}} \Omega_i$ is the set of joint observations, where $\Omega_i$ is the observation set of agent $i$; 
\item $\mathcal{T}(s,a,s'): \mathcal{S} \times \mathcal{A} \times \mathcal{S} \to  \{0,1\}$ is the deterministic transition function, mapping a state and joint action to a unique next state;
\item $\mathcal{O}(a,s',o): \mathcal{A} \times \mathcal{S} \times \Omega \to \{0, 1\}$ is the deterministic observation function, mapping a joint action and next state to a unique joint observation; 
\item $\mathcal{R}: \mathcal{S} \times \mathcal{A} \to \mathbb{R}$ is the immediate reward function; 
\item $\gamma \in (0, 1)$ is the discount factor for future rewards; 
\item $b_0 \in \Delta(\mathcal{S})$ is the initial belief, i.e., a probability distribution over initial states. 
\end{itemize*}
\end{definition}

Each agent $i$ selects its actions according to a local policy $\pi_i: (A_i \times \Omega_i)^* \to \mathcal{A}_i$, mapping its local action-observation history to an action.
In a Det-Dec-POMDP, uncertainty arises only from the initial state; thereafter, the system evolves deterministically according to $\mathcal{T}$ and $\mathcal{O}$.

In a single-agent Det-POMDP, the agent's belief about the true state becomes increasingly concentrated (the support of the belief monotonically decreases) as it gathers deterministic observations over time.
At each step, the agent can rule out states that are inconsistent with its action-observation history, gradually refining its belief until it converges to the true state.
Therefore, by exploiting both the deterministic dynamics and the concentrating nature of the belief, one can develop highly efficient planning methods for solving Det-POMDPs \cite{schutz2025finitestatecontrollerbasedoffline}.
However, even under deterministic dynamics, solving a Det-Dec-POMDP remains significantly more challenging than solving a Det-POMDP because each agent must reason not only about its own observations but also anticipate all other agents' possible histories and their induced behaviors, leading to a combinatorial explosion in the joint history and policy spaces.

\section{Iterative Deterministic POMDP Planning}

One major scalability bottleneck in existing state-of-the-art Dec-POMDP planning methods is the requirement to construct and reason over sufficient statistics \cite{szerMAAHeuristicSearch2012,dibangoye2016optimally}, such as distribution over these joint histories (also known as an occupancy state), to compute each agent’s optimal actions.
This involves evaluating an exponentially growing number of joint histories, which quickly becomes intractable as the problem size increases, especially in domains with thousands of observations per agent, even under finite-horizon settings.
Importantly, this issue persists in Det-Dec-POMDPs.
Although one might expect determinism to simplify planning, it does not alleviate the exponential growth in the joint history space.
This is because initial state uncertainty still leads to many possible observation sequences for each agent, resulting in a large number of possible joint histories that must be considered during computation—especially when a long horizon is required to complete the task.

In this work, we aim to efficiently address large-scale Det-Dec-POMDPs.
To tackle the scalability bottleneck, we propose a practical, optimized variant of the JESP approach \cite{you2021solving,you2023monte}, which finds Nash equilibrium solutions while leveraging recent advances in solving deterministic POMDPs efficiently.
%
% Building on this idea, we propose to adapt prior methods based on JESP, particularly for the infinite-horizon setting (InfJESP \cite{you2021solving} and MCJESP \cite{you2023monte}), where Dec-POMDPs are solved by iteratively computing each agent's best-response FSC policy until convergence.
%
% Specifically, in \Cref{sec:br-det-POMDP}, we prove that in a Det-Dec-POMDP, when the policies of all other agents (represented by FSCs) are fixed, agent $i$'s decision-making problem reduces to a deterministic POMDP.
% %
% Thus, in each iteration, we can efficiently compute agent $i$'s best-response policy by exploiting methods developed for solving Det-POMDPs.
% %
% \Cref{sec:main_algo} details our overall iterative deterministic POMDP planning framework for solving Det-Dec-POMDPs.
% %
% Finally, \Cref{sec:init} introduces the heuristic initialization method for generating each agent's initial policy.

\subsection{Best-Response Det-POMDP for Agent $i$}
\label{sec:br-det-POMDP}

We follow the same theoretical framework as InfJESP \cite{you2021solving}, where each agent $i$'s decision-making problem is formulated as a best-response POMDP when the FSC policies of the other agents are fixed.
Specifically, agent $i$ makes decisions by reasoning over a belief defined on an extended state space $\mathcal{E}$, where each $e^t \in \mathcal{E}$ is a tuple $\langle s^t, n^t_{\neq i}, \tilde{o}^t_i \rangle$.
\begin{theorem}
In a \textbf{Det-Dec-POMDP}, when the FSC policies of all agents except agent $i$ are fixed, the best-response model for agent $i$, denoted $\text{POMDP}_{\text{BR},i}$, is a \textbf{Det-POMDP}.
% \alex{should we reiterate that the FSC is deterministic?}
\end{theorem}

\begin{proof}
When the policies $\pi_{\neq i}$ of the other agents are fixed, 
the transition function of agent $i$'s best-response model $\text{POMDP}_{\text{BR},i}$ is given by:
\begin{align*}
& \mathcal{T}_e(e^t, a^t_i, e^{t+1}) = \sum_{o^{t+1}_{\neq i}}\mathcal{T}(s^t, \langle \psi_{\neq i}(n_{\neq i}^{t}), a^t_{i} \rangle, s^{t+1}) \\
&\cdot \mathbf{1}_{ n_{\neq i}^{t+1} = \eta_{\neq i}(n_{\neq i}^{t}, o^{t+1}_{\neq i}) } \cdot \mathcal{O}(s^{t+1}, \langle \psi_{\neq i}(n_{\neq i}^{t}), a^t_{i} \rangle, \langle o^{t+1}_{\neq i}, o^{t+1}_{i} \rangle)
\end{align*}
In this function, the actions of the other agents are deterministically chosen by $\psi_{\neq i}(n_{\neq i}^{t})$, so the first part, $\mathcal{T}$, corresponds to the transition function of the Det-Dec-POMDP, which deterministically maps the current state and joint action to the next state.
The second part, the transition of the other agents’ FSC nodes, is also deterministic by the FSC definition.
The third part, $\mathcal{O}(s^{t+1}, \langle \psi_{\neq i}(n_{\neq i}^{t}), a^t_{i} \rangle, \langle o^{t+1}_{\neq i}, o^{t+1}_{i} \rangle)$, is deterministic because, in a Det-Dec-POMDP, $\mathcal{O}$ maps the next state and joint action to a unique joint observation.
This implies that there exists only one possible $o^{t+1}_{\neq i}$, and thus the summation can be eliminated.
Therefore, the entire transition function $\mathcal{T}_e$ is deterministic, i.e., $\mathcal{T}_e(e^t, a^t_i, e^{t+1}) \in \{0,1\}$ and equals 1 for exactly one $e^{t+1}$.
Moreover, $\mathcal{O}_e$ is already defined as a deterministic observation function in $\text{POMDP}_{\text{BR},i}$'s formulation.
Thus, $\text{POMDP}_{\text{BR},i}$ is a deterministic POMDP with deterministic dynamics.
\end{proof}

This reduction allows solving a Det-Dec-POMDP as a sequence of Det-POMDPs, where each agent computes a best response to fixed policies of others, leveraging efficient Det-POMDP solvers that exploit deterministic structure.

\subsection{Main Algorithm}
\label{sec:main_algo}

The main procedure of \textit{Iterative Deterministic POMDP Planning} (IDPP) for solving Det-Dec-POMDPs is shown in \Cref{alg:main}.
This procedure is adapted from the (MC)JESP scheme, with key modifications highlighted in blue.
The process begins with a heuristic initialization step (\cref{alg_line:init}), which is described in detail in \Cref{sec:init}. It then enters an iterative best-response loop where agents update their policies until convergence.
In each iteration, IDPP constructs a best-response deterministic POMDP for the selected agent $i$ (\cref{alg_line:build_detpomdp}) and solves it efficiently using Det-MCVI \cite{schutz2025finitestatecontrollerbasedoffline} (\cref{alg_line:solve_detpomdp}), a solver specifically tailored for deterministic POMDPs.
This design leverages the structural determinism of Det-Dec-POMDPs to significantly improve computational efficiency.
Although the adaptation appears simple, it results in a powerful and scalable framework for solving large-scale multi-agent decision-making problems modeled as Det-Dec-POMDPs.

\begin{algorithm}
\caption{Main Algorithm for Solving Det-Dec-POMDP}
\label{alg:main}
\KwIn{Deterministic Dec-POMDP model $G$}
\KwOut{Nash equilibrium policy set $\{ \pi_i \}_{i \in \mathcal{I}}$}

\textcolor{blue}{$\{ \pi_i \} \gets \text{HeuristicInitFSCs}(G)$}\;  \label{alg_line:init} 
$V \gets \emptyset$\;

\While{$V$ not converged}{
    $i \gets \text{GetNextAgent()}$\;
    $\pi_{\neq i} \gets \text{FixOthersPolicies}(i, \{ \pi_j \})$\;
    \textcolor{blue}{$G_{\text{BR},i} \gets \text{BuildBRDetPOMDP}(G, \pi_{\neq i})$}\; \label{alg_line:build_detpomdp}
    \textcolor{blue}{$\pi_i^{*} \gets \textbf{SolveDetPOMDP}(G_{\text{BR},i})$}\; \label{alg_line:solve_detpomdp}
    $\pi_i \gets \pi_i^{*}$\;
    $V \gets \text{Evaluate}(G, \{ \pi_j \})$\;
}
\Return $\{\pi_j \}$
\end{algorithm}

\subsection{Heuristic Initialization}
\label{sec:init}

In JESP-style methods~\cite{nair2003taming, you2021solving, you2023monte}, initializing agents with non-myopic policies can significantly reduce iterations, especially when coordination is essential.  
Prior work (e.g., InfJESP and MCJESP) uses centralized heuristics that plan joint policies over the joint observation space, which becomes inefficient in large domains.

We propose a new heuristic that avoids joint observations by planning over each agent's local observation space.  
When other agents' policies are fixed, agent~$i$ plans using a Det-POMDP with extended state $e^t = \langle s^t, n^t_{\neq i}, \tilde{o}^t_i \rangle$ (\Cref{sec:br-det-POMDP}).
At initialization, we replace $n^t_{\neq i}$ with a default MDP policy $\pi_{\text{MDP}}$ that maps states to actions: $a^t_{\neq i} \gets \pi_{\text{MDP}}(s^t)$.
This yields a deterministic model $\text{POMDP}_{\text{init},i}$ with state $e^t = \langle s^t, \pi_{\text{MDP}}, \tilde{o}^t_i \rangle$, which agent~$i$ uses to compute its initial policy.  
Since all components evolve deterministically, $\text{POMDP}_{\text{init},i}$ is itself a Det-POMDP. The full initialization procedure is given in \Cref{alg:init}, and $\pi_{\text{MDP}}$ is easily computed using standard value iteration.

\begin{algorithm}
\caption{HeuristicInitFSCs}
\label{alg:init}
\KwIn{The Det-Dec-POMDP model $G$}
\KwOut{All agents' initial policies}
$\pi_{\text{MDP}} \gets \text{GetDefaultPolicyMDP}(G)$\;
$\{ \pi_{\text{init}}\} \gets \emptyset$ \;
\For{each agent $i \in \mathcal{I}$}
{
 $G_{\text{POMDP}_{\text{init},i}} \gets \text{BuildInitDetPOMDP}(G,i,\pi_{\text{MDP}})$\;
 $\pi_{\text{init},i} \gets \textbf{SolveDetPOMDP}(G_{\text{POMDP}_{\text{init},i}})$ \;
 $\{ \pi_{\text{init}}\}  \gets \{ \pi_{\text{init}}\} \cup$ $\pi_{\text{init},i}$\;
}

\Return $\{ \pi_{\text{init}}\} $
\end{algorithm}

It is important to note that this heuristic initialization does not guarantee optimality.  
Each agent assumes others follow a fixed MDP policy, which may yield suboptimal initial policies in tightly coordinated scenarios.  
However, since these serve only as starting points, the subsequent IDPP process (\Cref{alg:main}) iteratively refines them via best-response updates, eventually converging to a Nash equilibrium.

\section{Experiments}

\paragraph{Experiment Setting.} In this work, we introduce two Det-Dec-POMDP benchmarks: \textit{Multi-Agent Canadian Traveler Problem} (MACTP) and \textit{Collecting} as in \Cref{fig:envs}.
Each is configurable with parameters such as grid size and number of agents.
For each problem instance, we perform 10 runs for each algorithm to compute the average return and time used.
Each runs's joint policies are evaluated with $10^5$ episodes from a random starting state sampled from $b_0$.
All environments are initialized with the same random seed to ensure each algorithm is solving the exact same instance.

\begin{figure}[t]
    \centering
    \includegraphics[width=\linewidth]{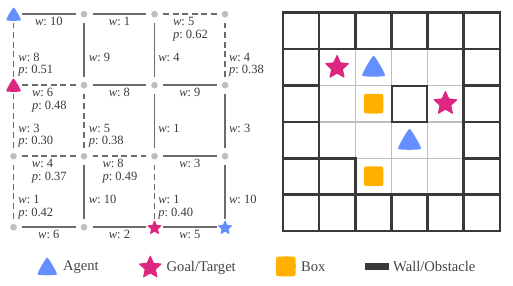}
    \caption{LEFT: An instance of MACTP$\langle 4, 2, 10\rangle$. RIGHT: An instance of Collecting$\langle 4, 4, 2, 2 \rangle$.}
    \label{fig:envs}
\end{figure}

\paragraph{Compared Algorithms.}
We compare the proposed method, IDPP, against the following baselines:

\begin{itemize}
\item \textbf{InfJESP} \cite{you2021solving} and \textbf{MCJESP} \cite{you2023monte}: State-of-the-art Dec-POMDP planners that compute Nash equilibrium solutions, representing infinite-horizon policies with finite-state controllers (FSCs). We use the official implementations provided by the authors.

\item \textbf{MAA*} \cite{szerMAAHeuristicSearch2012}: A heuristic search-based Dec-POMDP solver capable of computing optimal solutions for finite-horizon settings. We use the implementation from the MADP toolbox \cite{oliehoek2017madp}.

\item \textbf{Deterministic POMDP Heuristic}: The baseline described in \Cref{sec:init}, which decomposes the problem into independent Det-POMDPs for each agent, then solves each using DetMCVI \cite{schutz2025finitestatecontrollerbasedoffline} to produce individual FSCs.

\item \textbf{IQL} \cite{tan1993multi} and \textbf{MAPPO} \cite{Yu2021MAPPO}: Two popular MARL algorithms. Both are implemented in Python using PyTorch and employ LSTM-based policies to address partial observability.
\end{itemize}

All planning methods are implemented in C++, and a time limit of 10,000 seconds is imposed to evaluate their efficiency.
MARL methods adopt a fixed training budget of 10,000 episodes (each episode lasting up to 100 steps), rather than enforcing a time constraint.
All methods are evaluated based on their \textit{discounted accumulated rewards}.
The source code and parameter details are provided in the supplementary materials.

\subsection{Multi-Agent Canadian Traveler Problem}

The Multi-Agent Canadian Traveler Problem (MACTP) is generated by the tuple $\langle N, n_a, n_e\rangle$, where $N$ is the grid size, $n_a$ is the number of agents, and $n_e$ is the number of stochastic edges.
In MACTP, each edge $i$'s weight $d_i$ is randomly initialized from $\{1,\dots,10\}$.
Stochastic edges may be blocked or unblocked with a given probability; the edge's true status can only be observed when an agent reaches one of its incident vertices.
Stochastic edges and their blockage probabilities are also randomly initialized.
Goal vertices are assigned among the final $N^2 - n_e$ to $N^2$ nodes for each agent.
Each agent takes actions from $\mathcal{A}$, which is the action space consisting of $\{\uparrow, \rightarrow, \downarrow, \leftarrow, \circlearrowleft\}$, where $\circlearrowleft$ denotes a wait action.
The state space $\mathcal{S}$ combines the agents' positions and the edge states, resulting in $|\mathcal{S}| = (N^2)^{n_a} \times 2^{n_e}$.
At each time step, each agent observes the traversability of nearby edges and the locations of other agents.
A reward of $500$ is given when each agent reaches its corresponding goal position, and a cost equal to the edge distance $d_i$ is incurred for each successful movement action.
Therefore, the objective is for the agents to explore the environment and identify the shortest traversable paths to their goals.

In the MACTP domain, initial-state uncertainty creates a wide range of possible configurations.  
For instance, with two edges having blockage probabilities 0.3 and 0.4, the initial belief spans four blockage states (e.g., \(Pr(\text{both blocked}) = 0.12\), \(Pr(\text{edge1 blocked, edge2 not}) = 0.18\), etc.).  
Once blockage probabilities and edge weights are generated (using a fixed seed), they are known to the agents—only the blockage configuration remains uncertain.  
Although agents do not explicitly share observations, they infer environmental information by observing each other’s movements: an agent moving toward a vertex suggests traversable edges, while stopping or circling near an edge indicates possible blockage.  
This observational cooperation enables agents to indirectly gather knowledge and adapt their behavior accordingly.

\subsection{Collecting Problem}

The Collecting Problem is a Det-Dec-POMDP generated by the tuple 
$\langle H, W, n_a, n_b \rangle$,
where a team of $n_a$ agents operates on a structured grid of size $(H+2)\times(W+2)$, surrounded by untraversable wall cells. 
Within the interior $H \times W$ region, $n_b$ obstacle cells and $n_b$ goal cells are randomly placed, but their positions are fixed and known to all agents at the start of the problem.
The agents must cooperatively pick up and deliver $n_b$ indistinguishable boxes to the $n_b$ designated goal squares. 
Each successful delivery yields a reward of $+100$.
Agents choose actions from the set $\mathcal{A} := \{\uparrow, \rightarrow, \downarrow, \leftarrow, \circlearrowleft\}$.
Boxes are picked up or dropped automatically when an agent occupies the same cell.
To resolve potential conflicts, agent actions are executed in a fixed sequential order, making the problem asymmetric.
Each agent receives an observation $o_i$ comprising the $3\times3$ grid centered on its location, where each cell may indicate a wall, an empty space, a box, an agent, or a goal.
While the placement of obstacles and goal locations is fixed and known to agents, there is uncertainty regarding the initial state of the system. 
Specifically, agents' starting positions and the initial locations of the boxes are unknown, which introduces uncertainty in the initial belief.
This uncertainty is resolved as agents gather more information through their observations during executions.
The approximate state space size is:
\[
|\mathcal{S}| \approx (C \times 2)^{n_a} \times \binom{C}{n_b},
\]
where $C = H \times W - n_b$ is the number of free (e.g., non-wall, non-obstacle) cells, and the factor $2$ captures the binary carrying status of each agent.

In this problem, each agent can carry at most one box, and all boxes are indistinguishable, with any box deliverable to any goal location. 
Therefore, agents must coordinate to avoid redundant deliveries and resolve path conflicts, while considering the uncertainty from the initial state. 

\subsection{Results}
\label{sec:results}

% \yang{add some images of performance comparisons in each iteration...}

\begin{table*}[t]
\centering
\scriptsize
\setlength{\tabcolsep}{3pt} % adjust horizontal spacing
\renewcommand{\arraystretch}{1.1} % adjust vertical spacing
\begin{tabular}{@{}llccccccc@{}}
\toprule
& & \multicolumn{4}{c}{\textbf{MACTP} $\langle N, n_a, n_e \rangle$} & \multicolumn{3}{c}{\textbf{Collecting} $\langle W, H, n_a, n_t \rangle$} \\
\cmidrule(lr){3-6} \cmidrule(lr){7-9}
& \textbf{Instance} 
& $\langle3,2,5\rangle$ & $\langle4,2,8\rangle$ & $\langle4,2,12\rangle$ & $\langle5,2,14\rangle$ 
& $\langle4,3,2,2\rangle$ & $\langle4,4,2,3\rangle$ & $\langle5,5,2,4\rangle$ \\
\midrule
& $|\mathcal{S}|$   & $\sim 2.6\text{k}$ & $\sim 65.5\text{k}$ & $\sim 1.05\text{M}$ & $ > 10\text{M}$ & $\sim 4.7\text{k}$ & $\sim 52.9\text{k}$ & $\sim 2.82\text{M}$ \\
& $|b_0|$   & $2^5$ & $2^8$ & $2^{12}$ & $2^{14}$ & 30 & 112 & $\sim 2.73\text{k}$ \\
& $|\mathcal{O}_i|$   & 279 & 656 & 928 & $\sim 1.83\text{k}$ & 166 & 593 & $\sim 1.75\text{k}$ \\
\midrule

\multirow{2}{*}{IQL} 
& Return & $849.60 \pm 43.08$ & $696.33 \pm 45.92$ & $526.31 \pm 64.60$ & $455.62 \pm 83.22$ & $121.95 \pm 9.69$ & $205.44 \pm 11.64$ & $216.23 \pm 14.28$ \\
& Time   & -- & -- & -- & -- & -- & -- & -- \\
\midrule

\multirow{2}{*}{MAPPO} 
& Return & $808.84 \pm 39.10$ & $542.82 \pm 46.48$ & $427.97 \pm 42.19$ & $252.84 \pm 23.79$ & $125.17 \pm 7.60$ & $189.42 \pm 6.34$ & $197.63 \pm 5.17$ \\
& Time   & -- & -- & -- & -- & -- & -- & -- \\
\midrule

\multirow{2}{*}{Det-POMDP Heur.} 
& Return & $682.13 \pm 0.54$ & $705.57 \pm 8.26$ & $603.95 \pm 8.05$ & $516.84 \pm 35.77$ & $118.60 \pm 2.11$ & $197.92 \pm 2.33$ & $186.46 \pm 4.39$ \\
& Time   & $0.4 \pm 0.2$ & $8.2 \pm 1.4$ & $138.6 \pm 34.6$ & $710.3 \pm 46.8$ & $0.7 \pm 0.5$ & $15.3 \pm 1.1$ & $1064.8 \pm 109.8$ \\
\midrule

\multirow{2}{*}{InfJESP} 
& Return & $\textbf{918.37} \pm \textbf{0.28}$ 
& \multirow{2}{*}{\centering $^\ddagger$} 
& \multirow{2}{*}{\centering $^\dagger$} 
& \multirow{2}{*}{\centering $^\dagger$} 
& $\textbf{186.39} \pm \textbf{0.72}$ 
& \multirow{2}{*}{\centering $^\dagger$} 
& \multirow{2}{*}{\centering $^\dagger$} \\
& Time   & $23.8 \pm 1.4$ & & & & $44.6 \pm 6.4$ & & \\
\midrule

\multirow{2}{*}{MCJESP} 
& Return & $874.39 \pm 34.93$ & $743.76 \pm 70.83$ & $459.64 \pm 104.58$ & $531.13 \pm 82.86$ & $177.51 \pm 5.50$  & $252.97 \pm 16.51$ & $237.09 \pm 32.80$ \\
& Time   & $159.2 \pm 71.1$ & $197.6 \pm 41.0$ & $1381.3 \pm 130.4$ & $4115.8 \pm 659.6$ & $383.6 \pm 52.9$  & $628.8 \pm 42.3$ & $8509.6 \pm 684.3$ \\
\midrule

\multirow{2}{*}{\textbf{IDPP (Ours)}} 
& Return & $912.71 \pm 0.32$ & $\textbf{867.58} \pm \textbf{2.78}$ & $\textbf{798.47} \pm \textbf{2.40}$ & $\textbf{873.16} \pm \textbf{5.21}$ & $184.42 \pm 1.16$ & $\textbf{267.71} \pm \textbf{0.32}$ & $\textbf{315.56} \pm \textbf{2.81}$ \\
& Time   & $1.8 \pm 0.4$ & $15.3 \pm 2.6$ & $570.8 \pm 110.3$ & $1706.6 \pm 298.4$ & $2.4 \pm 0.5$ & $44.6 \pm 1.4$ & $4662.3 \pm 488.6$ \\
\bottomrule
\end{tabular}
\caption{
Performance (avg. return ± std) and computation time (in seconds) of algorithms on MACTP and Collecting instances. 
Each instance is defined by its structural parameters. 
$^\dagger$ indicates infeasibility due to memory constraints; $^\ddagger$ means infeasible due to time limits, and -- means not applicable.
}
\label{tab:alg_comparison}
\end{table*}

We first evaluate the optimal Dec-POMDP algorithm MAA*.
As a finite-horizon method, we test MAA* with horizons of 10 and 20.
However, we are unable to obtain a valid policy even for horizon 10 due to MAA*'s excessive memory consumption.
We track the memory usage of MAA* over the first 60 seconds and compare it with other planning methods on the problem $\text{MACTP}\langle3,2,5\rangle$.
As shown in \Cref{fig:memory_usage}, MAA* quickly exhausts memory and the program terminates shortly afterward.
In contrast, IDPP maintains consistently low memory usage, as it avoids reasoning over every possible joint history.
This suggests that building the sufficient statistics to compute optimal policies may be impractical in large (Det-)Dec-POMDPs.

\begin{figure}[t!]
    \centering
    \includegraphics[width=\linewidth]{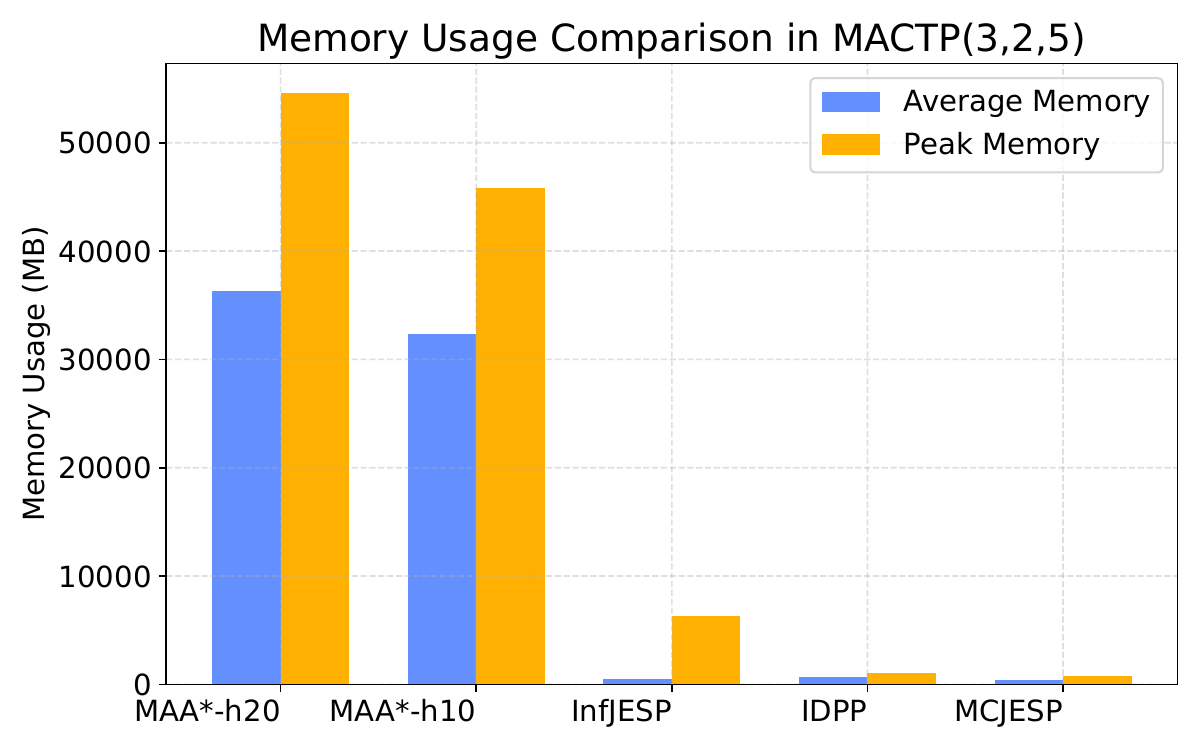}
    \caption{
    Comparison of average and peak memory usage (in MB) for different planning algorithms on the $\text{MACTP}\langle3,2,5\rangle$ problem.
    For MAA*, memory usage is recorded over the first 60 seconds, after which the program was terminated due to memory exhaustion in both horizon-10 and horizon-20 settings.
    For all other planners, memory usage is tracked until the problem is successfully solved.
    Algorithms are presented in descending order of peak memory usage.
    }
    \label{fig:memory_usage}
\end{figure}
\begin{figure}[t!]
    \centering
    \includegraphics[width=0.9\linewidth]{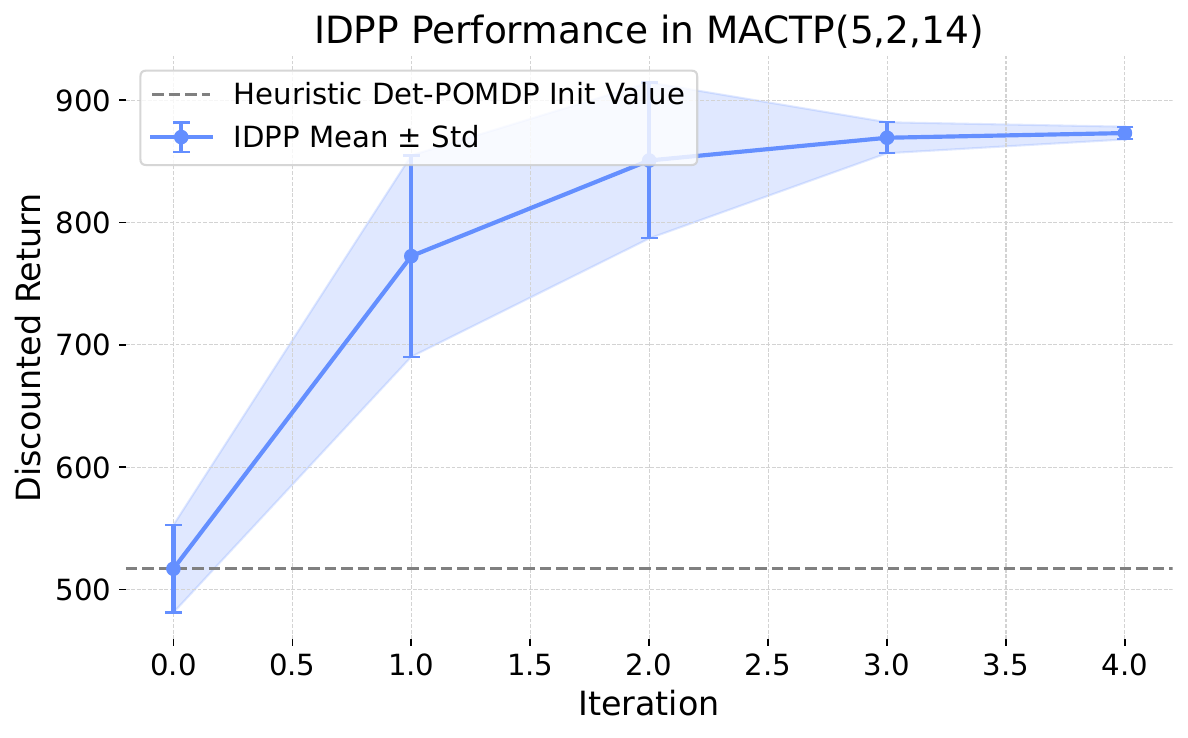}
    % \vspace{1mm}
    \includegraphics[width=0.9\linewidth]{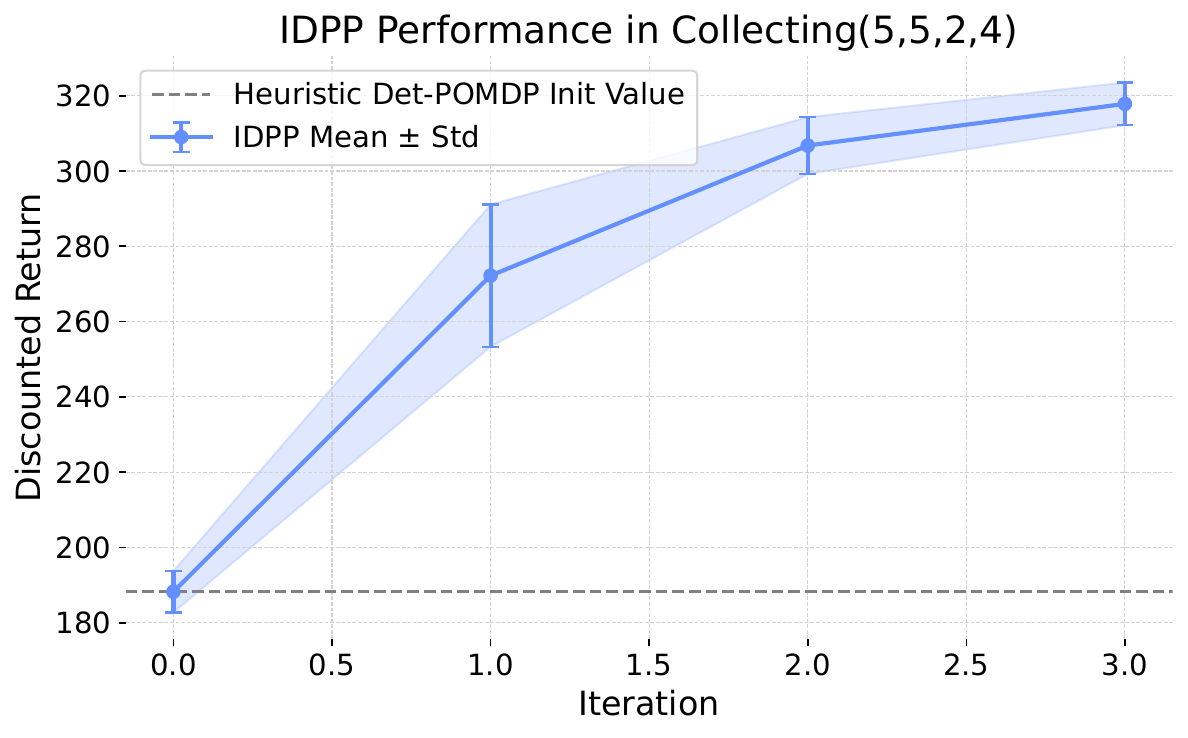}
    \caption{IDPP's performance across iterations in problem instances $\text{MACTP}\langle5,2,14\rangle$ and  $\text{Collecting}\langle5,5,2,4\rangle$.}
    \label{fig:convergence_plots}
\end{figure}

The performance of other algorithms is summarized in \Cref{tab:alg_comparison}.  
InfJESP outperforms existing methods on smaller problems (e.g., MACTP$\langle3,2,5\rangle$ and Collecting$\langle4,3,2,2\rangle$) by using the exact model and SARSOP to optimally compute each agent’s best response.  
However, it does not scale well to larger instances.  
MCJESP, InfJESP’s successor, achieves better scalability while maintaining good performance by constructing each agent’s FSC node-by-node using Monte Carlo planning (POMCP) within a fixed time budget (1 second per node in our experiments). More planning time may improve results further.  

Despite its simplicity, our heuristic initialization (\Cref{sec:init}) provides competitive performance relative to MCJESP at significantly lower computational cost.  
This demonstrates its effectiveness as a strong starting point for IDPP, which further improves solutions as shown in \Cref{fig:convergence_plots}.  
On large instances where InfJESP fails, IDPP consistently outperforms other methods with significant less computation time.  
By leveraging a deterministic POMDP solver in each iteration, IDPP achieves more accurate and efficient planning than MCJESP, leading to higher-quality Nash equilibrium policies.  
However, we note this advantage applies specifically to Det-Dec-POMDPs.  
%
% As shown in \Cref{fig:memory_usage}, MCJESP also uses less memory due to its incremental FSC construction, trading off some planning efficiency for scalability on large, complex problems.  

Finally, MARL methods MAPPO and IQL, relying solely on partial observations and sparse rewards, successfully learn policies for most tasks.  
This highlights the power of recurrent networks in handling partial observability.  
However, due to function approximation errors, their performance is generally inferior to model-based planners.  
Interestingly, IQL outperforms MAPPO in discounted return across most problems, despite MAPPO’s reputation as a state-of-the-art MARL method.  
Analysis reveals both succeed in task completion in most runs, but MAPPO tends to generate longer trajectories, lowering its discounted return.  
This may stem from (1) discrete action spaces favoring Q-learning methods like IQL, and (2) policy gradient methods like MAPPO emphasizing long-term optimization.

\section{Discussion of Contributions and Limitations}

In this article, we introduce the class of Deterministic Decentralized POMDPs (Det-Dec-POMDPs), a natural extension of Deterministic POMDPs \cite{bonetDeterministicPOMDPsRevisited2009} to the multi-agent setting.  
This model is also a further simplification of Quasi-Deterministic Dec-POMDPs \cite{besse2009quasi}, assuming deterministic observations. 
Such a framework is well suited for problems where uncertainty stems solely from the initial state, and both actions and observations are deterministic—such as high-level task planning in multi-robot systems, including some navigation and path planning applications. 
We then propose IDPP, a practical JESP variant aimed at efficiently solving large-scale Det-Dec-POMDPs.  
The main idea is intuitive and effective: ``choosing the right tool for the right subproblem," where we decompose the large Det-Dec-POMDP into a sequence of individual agents' Det-POMDPs to leverage powerful Det-POMDP solvers.  
As a result, IDPP becomes a highly efficient Det-Dec-POMDP solver that outperforms existing methods, to our knowledge, in this specific problem class.  
Moreover, we contribute two scalable benchmarks to facilitate research on scalability.

While IDPP is efficient for Det-Dec-POMDPs, it is not suitable for general Dec-POMDPs with stochastic transitions or observations.  
Therefore, it should be applied only when environment dynamics are deterministic.  
Another limitation is that our current IDPP implementation is single-threaded, so further speedups may be achieved through parallelization.

\section{Conclusion}

Many high-level robotic decision-making problems can be naturally modeled with deterministic actions and observations, where uncertainty primarily stems from the initial state.
Motivated by this, we introduce the Det-Dec-POMDP framework to capture such structure, along with IDPP, a practical solver adapted from a JESP variant for solving large Det-Dec-POMDPs to support future applications.
Our work may open a promising direction for planning in multi-agent partially observable domains where full stochasticity is unnecessary.

\section*{Acknowledgment}
This work has been supported by the EPSRC Energy Programme under UKAEA/EPSRC Fusion Grant 2022/2027 No. EP/W006839/1.
\bibliography{aaai2026}

\end{document}